\documentclass[11pt]{article}

\usepackage[preprint]{acl}

\usepackage{times}
\usepackage{latexsym}

\usepackage[T1]{fontenc}
\usepackage[utf8]{inputenc}

\usepackage{microtype}

\usepackage{inconsolata}

\usepackage{graphicx}

\usepackage[utf8]{inputenc} % allow utf-8 input
\usepackage[T1]{fontenc}    % use 8-bit T1 fonts

\usepackage{url}            % simple URL typesetting
\usepackage{booktabs}       % professional-quality tables
\usepackage{amsfonts}       % blackboard math symbols
\usepackage{nicefrac}       % compact symbols for 1/2, etc.
\usepackage{microtype}      % microtypography

\newcommand{\limfullfeed}{\ensuremath{\tt LimFullFeed}\xspace}
\newcommand{\nolookback}{\ensuremath{\tt NoLookBack}\xspace}
\newcommand{\allfeedback}{\ensuremath{\tt AllFeedback}\xspace}
\newcommand{\limbanfeed}{\ensuremath{\tt LimBanFeed}\xspace}

\newcommand{\bb}[1]{\mathbb{#1}}
\newcommand{\fl}[1]{\mathbf{#1}}
\newcommand{\ca}[1]{\mathcal{#1}}

\newcommand{\s}[1]{\mathsf{#1}}
\newcommand{\lr}[1]{\left|\left|#1\right|\right|}
\newcommand{\ip}[1]{\langle #1 \rangle}
\usepackage{algorithm}
\usepackage{algorithmicx}
\usepackage[noend]{algpseudocode}
\usepackage{amsmath}
\usepackage{amssymb}
\usepackage{amsthm}
\usepackage{bbm}
\usepackage{bm}
\usepackage{color}
\usepackage{dirtytalk}
\usepackage{dsfont}
\usepackage{enumerate}
\usepackage{graphicx}
\usepackage{listings}
\usepackage{mathtools}

\usepackage{soul}
\usepackage{subfigure}
\usepackage{times}
\usepackage{xspace}

\usepackage[dvipsnames]{xcolor}

\usepackage[capitalize,noabbrev]{cleveref}

\usepackage[textsize=tiny]{todonotes}

\theoremstyle{plain}
\newtheorem{theorem}{Theorem}[section]

\theoremstyle{definition}

\newtheorem{assumption}[theorem]{Assumption}
\theoremstyle{remark}

\DeclareMathOperator*{\argmax}{arg\,max\,}

\title{Online Learning with LLM Experts from Limited Feedback}

\author{Wang Wei$^1$, Soumyabrata Pal$^2$, 
Koyel Mukherjee$^2$, Franck Dernoncourt$^2$,\\
\bf{Ryan A. Rossi$^2$, Branislav Kveton$^2$, Hoda Eldardiry$^1$\thanks{Corresponding author.}}\\
$^1$Virginia Tech, $^2$Adobe Research\\
\texttt{\{wangwei718,hdardiry\}@vt.edu}
}

\begin{document}
\maketitle
\begin{abstract}
We study adaptive routing of prompts to large language model (LLM) experts to maximize response quality in an online setting with limited feedback. We formulate it as a bandit problem with $K$ actions that represent experts and $d$ features that encode prompts, over a horizon of $T$ rounds. We propose algorithms that strategically select and observe rewards to minimize regret. In the full-information setting, we achieve a regret of $\tilde{O}(d T / \sqrt{m})$, while in the bandit setting we achieve $\tilde{O}(d T \sqrt{K / m})$, where $m \ll T$ is a budget on feedback. Our experiments show that we efficiently learn high-quality routing strategies across diverse LLMs from limited feedback.
\end{abstract}

\section{Introduction}
\label{sec:introduction}

\emph{Large language models (LLMs)} are pervasive nowadays. OpenAI's GPT models \citep{OpenAI}, LLaMA3 \citep{dubey2024llama}, and Mistral AI \citep{MistralAI} have been successfully used to solve many tasks, such as document processing and code generation. Different LLMs have different costs and capabilities \citep{LLM-comparison}. The diversity of costs, even within the same family of models, can be stark \citep{Open-ai-pricing}. The diverse capabilities of LLMs are obvious from public datasets. For instance, the Nectar dataset \citep{zhu24starling7b} contains responses to $183$k prompts of many popular models judged by GPT-4. When GPT-3.5-Turbo, GPT-3.5-Turbo-Instruct, GPT-4, GPT-4-0613, LLaMA-2-7B-Chat, and Mistral-7B-Instruct are judged, the win rates of the models are $0.182$, $0.091$, $0.203$, $0.319$, $0.073$, and $0.132$. Therefore, no model dominates the others more than a third of the time and adaptation is beneficial.

We formulate the problem of online learning with \emph{LLM experts} as follows. We have $K$ different LLMs and interact with them sequentially over $T$ rounds. In each round, a prompt arrives and we route it to one expert, conditioned on the prompt. The expert responds and its response is associated with some reward, which is unobserved. This is because the responses of LLMs are generally not evaluated by their users. Our goal is to learn to route each prompt to the expert with the highest mean reward. This is impossible without feedback. Therefore, we make a realistic assumption that we get access to limited feedback that is as good as that from humans. This could be human feedback or a stronger LLM used as an \emph{LLM judge} \citep{li2024llms, LLM-judgement, llm-judge}. This feedback is expensive, either because of human labor or computation cost, and thus we can only use it $m \ll T$ times, where $m$ is determined by the available budget. The tradeoff between rewards and feedback is not clear a priori. Therefore, we maximize rewards under the constraint on feedback rather than a linear combination of the two quantities.
Finally, this problem is inherently online since user prompts are revealed only upon interaction, making offline solutions infeasible. Routing decisions must be made sequentially in real time, with no prior knowledge of the prompt distribution.

We solve our problem as a \emph{contextual bandit} \citep{langford08epochgreedy,li10contextual,lattimore19bandit} with $K$ experts, where each LLM is an expert and the context is an embedding of the prompt. The main difference from all prior work on contextual bandits is that only \emph{$m$ rewards out of $T$ can be observed}. The agent can decide \emph{what to observe} and \emph{when to observe it}, as long as it could have observed it before. The main challenge in the algorithm design is doing it at a near-optimal rate in a regret minimization setting. The control over what to observe and when to observe it differentiates our work from other bandit settings that involve partial observations and we discuss these extensively in \cref{sec:related work}. The online setting with limited feedback on LLM response quality differentiates our work from prior LLM optimization approaches, which are typically studied in offline settings or do not directly optimize response quality. We also discuss them in \cref{sec:related work}.
We make the following contributions:
\begin{enumerate}
  \item We study the full-information setting (\cref{sec:full-information}), where the agent can observe the rewards of all experts in any past round. The key idea in our algorithm is to progressively observe rewards of the experts with the highest information gain. The algorithm is computationally efficient and its regret is $\tilde{O}(d T / \sqrt{m})$, which matches our lower bound up to $\sqrt{d}$. Note that the bound is $\tilde{O}(\sqrt{T})$ when $m = \Omega(T)$.
  \item We also study the bandit setting (\cref{sec:bandit}), where the agent can only observe at a certain round the reward of the expert who has generated the response for that round. The key idea in our algorithm is to progressively observe rewards of the experts with the highest information gain, separately for each LLM expert. The algorithm is computationally efficient and its regret is $\tilde{O}(d T \sqrt{K / m})$. The additional $\sqrt{K}$ factor comparing to the full-information setting is due $K$ times less feedback.
  \item We extend the bandit setting to varying expert costs where the price of evaluating the responses of experts is non-uniform. Due to space constraints, this result has been moved to  \cref{sec:variable-cost bandit}.
  \item We evaluate all algorithms empirically on the Nectar dataset \citep{zhu24starling7b} and show that they can learn a high-quality routing agent for many popular LLMs experts, such as GPT-3.5, GPT-4, LLaMA, and Mistral (\cref{sec:experiments}).
\end{enumerate}

\textbf{Technical Novelty:} A key challenge in our analysis is bounding the reward gap between the optimal and selected experts for prompts without feedback. Unlike standard linear bandits, limited feedback prevents immediate updates to the covariance matrix, making it difficult to bound regret using standard techniques. To address this in the full-information setting, we introduce a novel feedback strategy: at regular intervals, we select past prompts that maximize the determinant of the covariance matrix. This approach quickly reduces uncertainty in important directions and allows us to bound regret. We further show that periodic feedback selection performs nearly as well as hindsight-optimal choices. In the bandit setting, we extend this approach by maintaining separate confidence sets per expert to ensure accurate regret guarantees.

\textbf{Outline:} In \cref{sec:setting}, we state our problem and define our model. In \cref{sec:full-information,sec:bandit}, we describe our algorithms for the full-information and bandit settings, respectively, along with providing theoretical guarantees. 
In \cref{sec:variable-cost bandit}, we extend our bandit algorithm to the setting with varying costs of evaluating experts.
In \cref{sec:experiments}, we provide empirical results using our algorithms.
Finally, in \cref{sec:proofs}, we provide detailed proofs for all our results.

\section{Problem Setting}
\label{sec:setting}

\textbf{Notation:} We denote by $[K]$ the set $\{1,2,\dots,K\}$. We denote scalars and vectors by lowercase letters (say $x$). We denote matrices and fixed global parameters by capital letters (say $X$). For a vector $x$, $x_i$ denotes its $i^{\s{th}}$ entry. For an indexed vector $x_j$, $x_{j,i}$ denotes its $i^{\s{th}}$ entry. We let $\lr{x}_{A}$ be the weighted 2-norm $\sqrt{x^{T}Ax}$ with respect to a positive semi-definite matrix $A$. We define the corresponding inner product as $\ip{x,y}_{A} = x^{T}Ay$. $0_d$ and $I_d$ are the zero vector and identity matrix in $d$ dimensions, respectively. $\ca{N}(0_d,\Sigma)$ is the Gaussian distribution in $d$-dimensions with zero mean and covariance matrix $\Sigma$. For a matrix $A$, we write $A_i$ to denote its $i^{\s{th}}$ column. Let $\ca{B}^d \equiv \{x\in \bb{R}^d \mid \lr{x}_2 \le 1\}$ be the unit ball in $d$ dimensions.

\textbf{Problem Formulation:} We introduce our setting as a variant of a classic linear bandit \citep{abbasi-yadkori11improved}. We have $T$ rounds and $K$ LLM \emph{experts}. Each expert is indexed by $a\in [K]$ and associated with an unknown \emph{parameter vector} $\theta_a\in \ca{B}^{d}$. At each round $t\in [T]$, a \emph{prompt} arrives and we denote it by $x_t \in L\cdot\ca{B}^d$ for some $L > 0$. 
The prompt is then treated as \emph{context} at round $t$. 

Given a prompt $x_t \in \bb{R}^d$, the algorithm chooses an expert $a_t \in [K]$ and obtains its response. In the linear model, the \emph{expected reward} for the response of expert $a$ in round $t$ is $\ip{\theta_a, x_t}$. We denote the vector of all stochastic \emph{rewards} in round $t$ by $r_t = (r_{t, a})_{a = 1}^K$ and define each reward as
\begin{align}
  r_{t,a}
  = \ip{\theta_{a},x_t}+ \eta_{a,t}\,.
  \label{eq:reward}
\end{align} 
We assume that the noise $\{\eta_{a,t}\}_{a\in [K], t\in [T]}$ is independent, both across the experts and rounds, and sub-Gaussian with a variance proxy $1$. The linear model is realistic since the reward is computed as a 
linear head on top of a frozen transformer embedding. Specifically, $x_t$ is the embedding of the prompt produced by a transformer, and $\theta_a$ is the linear head for expert $a$. This is standard in reward modeling, with the only distinction being that we do not fine-tune the embeddings.

Further, at each round $t$, the context $x_t$ arrives, the algorithm selects 
an expert $a_t$, and feedback is optionally requested afterward. 
This ordering reflects a key constraint: feedback is expensive and the decision of whether to query it is made after the action, as part of the exploration strategy.
One may ask whether prompt evaluations could instead be performed before selecting $a_t$, by looking back at past contexts similar to 
$x_t$ and leveraging their outcomes to make a more informed decision. 
However, doing so would require additional feedback queries at every 
round, violating the budget constraint $m \ll T$. Our formulation 
instead leverages past feedback frugally: routing decisions are made 
with whatever has been learned up to round $t$, without retroactive 
evaluation of past prompts triggered by the current context. We 
consider this alternative as the \emph{NoLookBack} baseline in our 
experiments.

\textbf{Limited Feedback:} In our setting, the feedback is expensive \citep{li2024llms}, due to time or monetary constraints\footnote{The cost and latency of using LLMs as judges depends on the token counts and pricing of the service provider, such as Azure Databricks \citep{llm-judge-metrics, Agent-eval-cost}, or on the computational resources available to host such LLMs on premise.}, and thus limited. In the classic linear bandit, the noisy reward is typically observed partially or completely at each round $t$. \textit{The key difference in our setting is that the reward vector generated at any round $t$ is not immediately observed.} However, feedback—provided as observations of rewards—is crucial for improving the selection of experts over successive rounds \citep{lattimore19bandit}.

The algorithm has a budget of $m < T$ observations, meaning that it can observe at most $m$ rewards, either whole vectors or its entries, across the entire time horizon. Crucially, the algorithm can adaptively decide at which rounds to obtain feedback. If, at round $t \in [T]$, the algorithm decides to collect feedback, it may choose any round $s_t \leq t$ and observe the noisy reward vector $r_{s_t}$ or its entry $r_{s_t, a_{s_t}}$. Many prior works in the bandit literature studied limited feedback (\cref{sec:related work}). The main difference in our setting is that the algorithm not only selects an expert at each round but also chooses, exploiting the problem structure, when and for which past prompts to collect feedback. 

The expert with the highest expected reward in round $t$ given the prompt $x_t$ is
\begin{align*}
  a^{\star}_t
  = \displaystyle \argmax_{a\in [K]}\ip{\theta_a,x_t}\,.
\end{align*}
We define the \emph{cumulative regret} in $T$ rounds as
\begin{align}\label{eq:reg_simple}
  \s{Reg}(T)
  = \bb{E}\Big[\sum_{t=1}^{T} \ip{\theta_{a^{\star}_t},x_t} - \sum_{t=1}^{T} \ip{\theta_{a_t},x_t}\Big]\,,
\end{align}
where the expectation is with respect to the randomness of the algorithm. In the remainder of the paper, we study different forms of limited feedback that can be obtained in practice and analyze regret in these settings.

\section{Full-Information Feedback}
\label{sec:full-information}

We start with the full-information setting, where the agent can observe rewards of all experts at any past round. While this setting is simpler than the bandit setting in \cref{sec:bandit}, it already exhibits basic properties of its algorithm design, that the problem can be solved by choosing observations with the highest information gain at regular time intervals. The former guarantees sub-linear regret and the latter allows us to trivially satisfy the observation budget.

\subsection{Algorithm}
\label{sec:limfullfeed}

\begin{algorithm*}[t]
\small
  \caption{\limfullfeed: Limited full-information feedback for expert selection.}
  \label{alg:full_info}
  \begin{algorithmic}[1]
    \State Initialize $b \gets m / T$, $V_0 \gets \lambda I_d$, $\mathcal{S}_0 \gets \emptyset$; and $y_{0, a} \gets 0_{d}$, $\widehat{\theta}_{0, a} \gets 0_d$ for all $a \in [K]$
    \For{round $t =1, \dots, T$}
      \State Obtain prompt $x_t$ and choose expert
      \begin{align}
        \label{eq:estimate}
        \textstyle
        a_t = \argmax_{a\in [K]} \ip{\widehat{\theta}_{t-1,a}, x_t}
      \end{align} 
      \If{$\lfloor b(t-1) \rfloor < \lfloor bt \rfloor$}
        \State Find most informative past observation
        \begin{align}
          \label{eq:det}
          \textstyle
          s_t = \argmax_{\ell \in \, [t]\setminus \mathcal{S}_{t-1}\, } \s{det}( V_{t-1} + x_{\ell}x_{\ell}^{T}) 
        \end{align} 
        \State Observe reward vector $r_{s_t}$ and update all statistics
        \begin{align}
          V_t & \gets V_{t - 1} + x_{s_t} x_{s_t}^T\,, \quad
          \mathcal{S}_t \gets \mathcal{S}_{t - 1} + \{s_t\}
          \nonumber \\
          \forall a \in [K]:
          y_{t, a} & \gets y_{t - 1, a} + r_{s_t, a} x_{s_t}\,, \quad
          \widehat{\theta}_{t, a} \gets V_t^{-1} y_{t, a}
          \label{eq:confidence}
        \end{align}
      \Else 
        \State Update $V_t \gets V_{t-1}$, $S_t \gets S_{t-1}$; and $y_{t,a} \gets y_{t-1,a}$, $\widehat{\theta}_{t,a} \gets \widehat{\theta}_{t-1,a}$ for all $a\in [K]$
      \EndIf
    \EndFor
  \end{algorithmic}
\end{algorithm*}

Our algorithm \limfullfeed is presented in \cref{alg:full_info} and we describe it next. The inputs are $K$ LLM experts, the number of rounds $T$, the feedback budget $m$, and a hyperparameter $\lambda$. We also initialize all statistics for tracking reward models of all experts online (\cref{sec:setting}), such as the common covariance matrix $V_0$ and \emph{ordinary least squares (OLS)} estimates $\widehat{\theta}_{0, a}$ for all experts.

In any round $t \in [T]$, \limfullfeed observes a prompt $x_t$ and chooses the best expert $a_t$ based on its estimated mean reward in \cref{eq:estimate}. \limfullfeed may decide to obtain feedback. The feedback is obtained when $\lfloor b (t - 1) \rfloor < \lfloor bt \rfloor$ holds. Roughly speaking, this happens every $T / m$ rounds since $b = m / T$. By following this strategy, we trivially guarantee that the observation budget constraint is satisfied. When the algorithm decides to obtain feedback, it can select any past unobserved round. Any round can be observed at most once because two repeated observations would be identical and hence not independent. More formally, we denote the set of past rounds where the feedback was previously obtained by $S_t \subseteq [t]$ and let the algorithm choose any round in $[t] \setminus S_t$. We denote the chosen round by $s_t \in [t] \setminus S_t$ and the observed rewards by $r_{s_t}$. Since we are in the full-information setting, $r_{s_t}$ is a vector of the rewards of all experts. After the feedback is obtained, all statistics are updated, such as the OLS estimates for all experts, $\widehat{\theta}_{t,a}$ in \cref{eq:confidence}.

The technical novelty in our algorithm design is in how $s_t$ is chosen. We consider all past prompts $\{x_s\}_{s \leq t}$ and choose the one that maximally increases the determinant of the covariance matrix in \cref{eq:det}. The intuitive idea behind this choice is that this increases all eigenvalues of the covariance matrix uniformly and thus leads to uniformly decreasing confidence intervals in all previously observed directions, encoded by the embeddings of the prompts. In turn, this yields sub-linear regret.

We would like to comment on two more aspects of \limfullfeed. First, the best expert in \cref{eq:estimate} is chosen using the mean reward estimate. This is because in the full-information setting, all experts are trained on the same past prompts and hence have the same covariance matrices. In the bandit setting (\cref{sec:bandit}), we account for non-uniform data collection across the experts. Second, a naive implementation of \limfullfeed has a $O(d^3)$ per-round time complexity, due to inverting $d \times d$ matrices and computing their determinants. This can be reduced to $O(d^2)$ by using the Sherman-Morrison formula for the former and the matrix determinant lemma for the latter.

\subsection{Main Results}

Our goal is to minimize regret under the constraint of obtaining feedback at most $m \leq T$ times. The constraint is satisfied trivially (\cref{sec:limfullfeed}). Therefore, we only need to prove a regret bound. We start by borrowing standard assumptions from linear bandit analyses \citep[Chapter 19]{lattimore2020bandit}.

\begin{assumption}
\label{assum:magnitude} All expert parameters $\{\theta_a\}_{a \in [K]}$ satisfy $\lr{\theta_a}_2 \le 1$. All prompts $\{x_t\}_{t = 1}^T$ satisfy $\lr{x_t}_2 \le L$. We also assume that $\max_{a,b\in [k], t\in [T]} \langle \theta_{a}-\theta_b, x_t \rangle \le 1$. 
\end{assumption}

\begin{theorem}
\label{thm:full_information} Choose any $K$ experts, $d$ features, horizon $T$, budget $m < T$, and $\lambda > 0$. Suppose that \cref{assum:magnitude} holds. Then the regret of \limfullfeed is $\s{Reg}(T) = O(dTm^{-1/2}\log (TL))$.
\end{theorem}

Due to space constraints, we only sketch the proof. The detailed proof is in \cref{app:detailed_full_info}

 \begin{proof}[Proof Sketch]
     Consider a fixed expert $a\in [K]$ with unknown $\theta_a$ and a prompt $x_t$ at round $t$. Using the OLS estimate $\widehat{\theta}_{t,a}$ for $\theta_a$, the error in estimated mean reward given by $\ip{\theta_{a}-\widehat{\theta}_{t,a},x_t}$ scales as $\lr{x_t}_{V_t^{-1}}$. This error can be rewritten as $\log \s{det}(V_t+x_tx_t^{T})-\log \s{det} V_t$. In standard linear bandits, we update the covariance matrix $V_{t+1}=V_t+x_tx_t^{T}$. Hence, we can add up the error over all rounds to obtain a telescoping sum. However, in our setting with limited feedback, $V_t$ is not updated at every round and therefore the above does not hold. The key idea in our proof stems from Line 7 in \limfullfeed where we look back to update $V_t$ with the prompt that increases its determinant most. Hence we can still bound $\log \s{det}(V_t+x_tx_t^{T})-\log \s{det} V_t$ from above by $\log \s{det} V_{t+1}-\log \s{det} V_t$. This upper bound is identical for all the rounds between two consecutive feedback and associated updates to the covariance matrix. Finally, we solicit feedback at regular intervals with interval length of $T/m$ rounds. Therefore we still get a telescoping sum after adding the errors but the limited feedback leads to an extra multiplicative factor of $T/m$ in the telescoping sum. This leads to an additional multiplicative factor of $\sqrt{T/m}$ in the regret bound. 
 \end{proof}

Note that regret does not depend on the number of experts $K$ since in the full-information setting, we obtain responses for all experts jointly. Therefore, when $m=T$, we can obtain feedback for all rounds and the regret guarantee that is achieved is $O(d\sqrt{T}\log T)$. This is reminiscent of the standard regret guarantee achieved for linear bandits \citep{lattimore2020bandit}. The additional cost of limited feedback arises in the form of a multiplicative factor of $\sqrt{T/m}$ leading to higher cost with lesser feedback. We also prove the following lower bound (detailed proof is in \cref{app:lower_bound}).

\begin{theorem}\label{thm:lower_bound_full_info}
    Consider the online expert selection problem with limited full information feedback, $d$ features, $K$ experts, horizon $T$, 
    feedback budget of $m<T$. Suppose all observations are Gaussian random variables with noise variance $1$. Let prompts $\{x_t\}_{t=1}^{T}\in \{-1,+1\}^d$ and $\{\theta_a\}_{a\in [K]}\in \{-\delta,+\delta\}^d$ for $\delta=(md)^{-1/2}$. Then there exists an instance such that the regret incurred must satisfy
    \begin{align*}
        \s{Reg}(\s{T}) \ge \frac{T\sqrt{d}\exp(-4)}{8\sqrt{m}}.
    \end{align*}
\end{theorem}

\textbf{Discussion:} Note that there is a gap of $\sqrt{d}$ between the upper and lower bounds. This gap stems from the fact that we do not know the prompts in advance and therefore, at all rounds $t$ we need to bound $\ip{\theta_{a}-\widehat{\theta}_{t,a},x}$ for all possible prompts $x\in \bb{R}^d$ and all experts $a\in [K]$. Such bounds are obtained in online settings with correlated observed random variables via a tail inequality on self-normalized martingales \cite{abbasi2012online}. 

Intuitively, ensuring the error bound is small for all vectors in the $d$-dimensional space leads to a union bound over $d$ dimensions that in turn leads to the additional $\sqrt{d}$ factor in the upper bound. If on the other hand, the prompts $\{x_{t}\}_{t=1}^{T}$ was known, then we would only need a union bound over $\s{T}$ vectors instead of all vectors in $\bb{R}^d$. This removes the additional $\sqrt{d}$ factor from the regret upper bound to make it tight up to logarithmic factors. 

Next we argue that obtaining feedback at regular intervals leads to an optimal regret bound in $m$. Suppose that the algorithm knew the prompts $\{x_{t}\}_{t=1}^{T}$ in advance and could also decide the order in which to obtain feedback. Then the algorithm would choose a subset $S$ of $m$ most-informative prompts, obtain feedback, and learn the expert parameters from them. The regret of this approach would be $O(T \max_{t\in [T]}\lr{x_t}_{V^{-1}})$, where $\max_{t\in [T]}\lr{x_t}_{V^{-1}}$ is the maximum confidence interval width and $V=\sum_{i\in S} x_ix_i^T$. An optimal solution to this problem is known as the G-optimal optimal design \citep[Chapter 21]{lattimore2020bandit} and its maximum confidence interval width is $O(\sqrt{d / m})$. This leads to a regret of $O(T \sqrt{d / m})$ over $T$ rounds and completes our argument.

We can consider a simpler algorithm which does not look back and use past prompts. 

The algorithm requests feedback for the input prompts at the rounds it decided to obtain feedback. Notice that such a baseline algorithm (we will call \nolookback) might suffer linear regret if an adversary provides prompts at the feedback rounds that are in an orthogonal subspace to the prompts in remaining rounds.

\section{Bandit Setting} 
\label{sec:bandit}
In this setting, we consider bandit feedback. This means that unlike the full information setting, here an agent can observe the reward of only one expert, that which was chosen to generate the response of a prompt at any past round.  Each expert might receive feedback a different number of times (unlike in the full-information setting). 

We are able to guarantee sub-linear cumulative regret in this setting as well, while satisfying the budget on feedback trivially, by algorithm design.

\begin{algorithm*}[t]
\small
\caption{\limbanfeed: Limited bandit feedback for expert selection.}\label{alg:banditv2}
\begin{algorithmic}[1]
    \State Initialize $z \gets T/m$, $V_{0,a} \gets \lambda I_d$, $\mathcal{S}_{0, a} \gets \emptyset$; and $y_{0, a} \gets 0_{d}$, $\widehat{\theta}_{0, a} \gets 0_d$ for all $a \in [K]$.

\For{rounds $t=1,2,\dots,T$}
\State Obtain prompt $x_t$, choose expert $a_t$ and increase its counter by 1. 
    \begin{align}\label{eq:estimatev2}
    a_t &= \displaystyle \argmax_{a \in [K]} \ip{\widehat{\theta}_{t-1,a}, x_t} + \beta \lr{x_t}_{V_{t-1,a}^{-1}}\\
    & n_{a_t} \leftarrow n_{a_t}+1\nonumber   
    \end{align} 
    
\If {$n_{a_t} \ge z$} 
\State Find the most informative past observation for the expert $a_t$. Reset its counter. 
\begin{align}\label{eq:detv2}
        s_t &= \displaystyle \argmax_{s\in  \{s \in [t]: a_s = a\}\setminus S_{t-1,a}} \s{det}( V_{t-1,a_t} + x_{\ell}x_{\ell}^{T})\\
        & n_{a_t} \gets 0 \nonumber
\end{align} 

\State Observe reward $r_{s_t,a_t}$ for the expert $a_t$ and update the corresponding statistics for $a_t$.
   \begin{align}
          V_{t, a_t} & \gets V_{t-1, a_t} + x_{s_t} x_{s_t}^T\,, \quad
          \mathcal{S}_{t, a_t} \gets \mathcal{S}_{t-1, a_t} + \{s_t\}
          \nonumber \\
          \forall a \in [K]:
          y_{t, a_t} & \gets y_{t - 1, a_t} + r_{s_t, a_t} x_{s_t}\,, \quad
          \widehat{\theta}_{t, a_t} \gets V_{t, a_t}^{-1} y_{t, a_t}
          \label{eq:confidence}
        \end{align}

\Else
\State Update $V_{t,a} \leftarrow V_{t-1,a}$, $y_{t,a} \leftarrow y_{t-1,a}$, $\widehat{\theta}_{t,a}=\widehat{\theta}_{t-1,a}$ and $S_{t,a} \leftarrow S_{t-1,a}$ for all experts $a\in [K]$ .
\EndIf
\EndFor
\end{algorithmic}
\end{algorithm*}

\subsection{Algorithm}
Here we describe our proposed algorithm \limbanfeed (\cref{alg:banditv2}) that chooses experts to generate response in the bandit feedback setting. 
Since different experts can potentially get feedback  different number of times, \limbanfeed adaptively chooses the experts and the corresponding past prompts (routed to them) to collect the feedback. The key idea is to observe rewards for those experts with highest information gain, by appropriately considering confidence bounds in reward estimates.

\limbanfeed takes as input the same parameters as \limfullfeed along with the additional hyperparameter $\beta$ corresponding to the confidence width. In the bandit setting, we initialize the covariance matrix $V_{0,a}$ along with other hyperparameters separately for each expert $a \in [K]$. In Line 2, we also initialize a variable $z$ which is set to be the average number of rounds between feedback, namely $T/m$. 

At each round $t$, the prompt $x_t$ arrives as context. Given the input prompt $x_t$, \limbanfeed chooses the expert $a_t$ in \eqref{eq:estimatev2} to generate the desired response by computing the upper confidence bound on the reward for each expert, given the OLS parameter estimate, and picking the one with the highest upper confidence reward. 
We maintain a counter $n_a$ for every expert $a\in [K]$ that keeps track of the number of times an expert is chosen for generating response. Once an expert has been used $z$ times, feedback is solicited for that expert by choosing a past prompt (it had served) appropriately and the counter is reset. 

Clearly, experts that are chosen more frequently will get more feedback.
We show here that the choice of $z$ helps respect the overall feedback budget of $m$ rounds.  The number of rounds where an expert $a\in [K]$ has been used to generate a response in the time horizon $T$ is $\left|S_{T,a}\right|$, hence the number of times the $a^{\s{th}}$ expert has been evaluated is $\lfloor \left|S_{T,a}\right|/z \rfloor$. 

Setting $z\geq \frac{T}{m}$ helps satisfy the feedback budget $m$: 
\begin{align*}
   \sum_{a\in [K]} \lfloor \frac{\left|S_{T,a}\right|}{z} \rfloor \le \sum_{a\in [K]} \frac{\left|S_{T,a}\right|}{z} = \frac{T}{z} \le m \implies z \ge \frac{T}{m}.
\end{align*}

For getting feedback for an expert $a$, \limbanfeed chooses a prompt index among the past prompts that maximizes the increase in the determinant of the covariance matrix for the expert $a$ (see \eqref{eq:detv2}). We only choose prompts that have not yet been evaluated. 

We update the OLS estimate of the parameter vector $\widehat{\theta}_{t,a}$ for expert $a_t$  (see \eqref{eq:confidence}) by using the observed scalar reward $r_{s_t,a}$ for the chosen prompt. 

\subsection{Main Results}

We show our main result below:

\begin{theorem}\label{thm:bandit}
Consider the online expert selection problem with limited bandit feedback, $K$ experts, $d$ features, horizon $T$ and feedback budget $m<T$. Suppose Assumption \ref{assum:magnitude} is true. Then for any $\lambda>0$ and $\beta=\sqrt{\lambda}+\sqrt{6\log T+d\log (1+TL^2/d)}$, the regret of \limbanfeed is
\begin{align*}
\s{Reg}(T) = O(dTK^{1/2}m^{-1/2}\log (TL))   
\end{align*}

\end{theorem}

The detailed proof is provided in Appendix \ref{app:bandit}, but we discuss our key ideas here. 

\begin{proof}[Proof Sketch]
    In the bandit setting, the covariance matrix for each expert is updated separately. However, the key idea for strategically choosing the prompt in history, given the expert, remains the same as in the full-information setting. To obtain feedback at round $t$ for expert $a$, we look back and choose the prompt for which the determinant of the covariance matrix increases the most, allowing us to suitably bound the error $\lr{x_t}_{V_{t-1,a}^{-1}}$ from above. Additionally, in our analysis, we use the fact that each expert is updated after being used z = $\lceil T/m \rceil$ times. This ensures that the error from an expert decreases over time, with the rate depending on how frequently the expert is invoked.
\end{proof}

Note that the number of feedback observations in the full-information setting is $K$ times that of the bandit setting for the same number of feedback rounds. 
Therefore, roughly speaking, in contrast to the full-information setting, the regret guarantee with the bandit feedback has an additional multiplicative factor of $\sqrt{K}$.

Further, when $m=T$ that is, we can obtain feedback for all rounds, Algorithm \ref{alg:banditv2} achieves a regret guarantee of $O(d\sqrt{TK}\log (TL))$. As in the full-information setting, the cost of limited feedback is a multiplicative factor of $\sqrt{T/m}$. Next, we prove a lower bound on the regret in the bandit feedback setting (detailed proof is in \cref{app:lower_bound}):

\begin{theorem}\label{thm:lower_bound_bandit}
    Consider the online expert selection problem with limited bandit feedback, $d$ features, $K$ experts, horizon $T$, 
    feedback budget of $m<T$. Suppose all observations are Gaussian random variables with noise variance $1$. Let prompts $\{x_t\}_{t=1}^{T}\in \{-1,+1\}^d$ and $\{\theta_a\}_{a\in [K]}\in \{-\delta,+\delta\}^d$ for $\delta=(md)^{-1/2}$. Then there exists an instance such that the regret incurred must satisfy
    \begin{align*}
        \s{Reg}(\s{T}) \ge \frac{TK\exp(-4)}{8\sqrt{m}}
    \end{align*}
\end{theorem}
\begin{figure*}[t!]
    \centering
    \subfigure[Full Information Setting]{
    \includegraphics[scale=0.3]{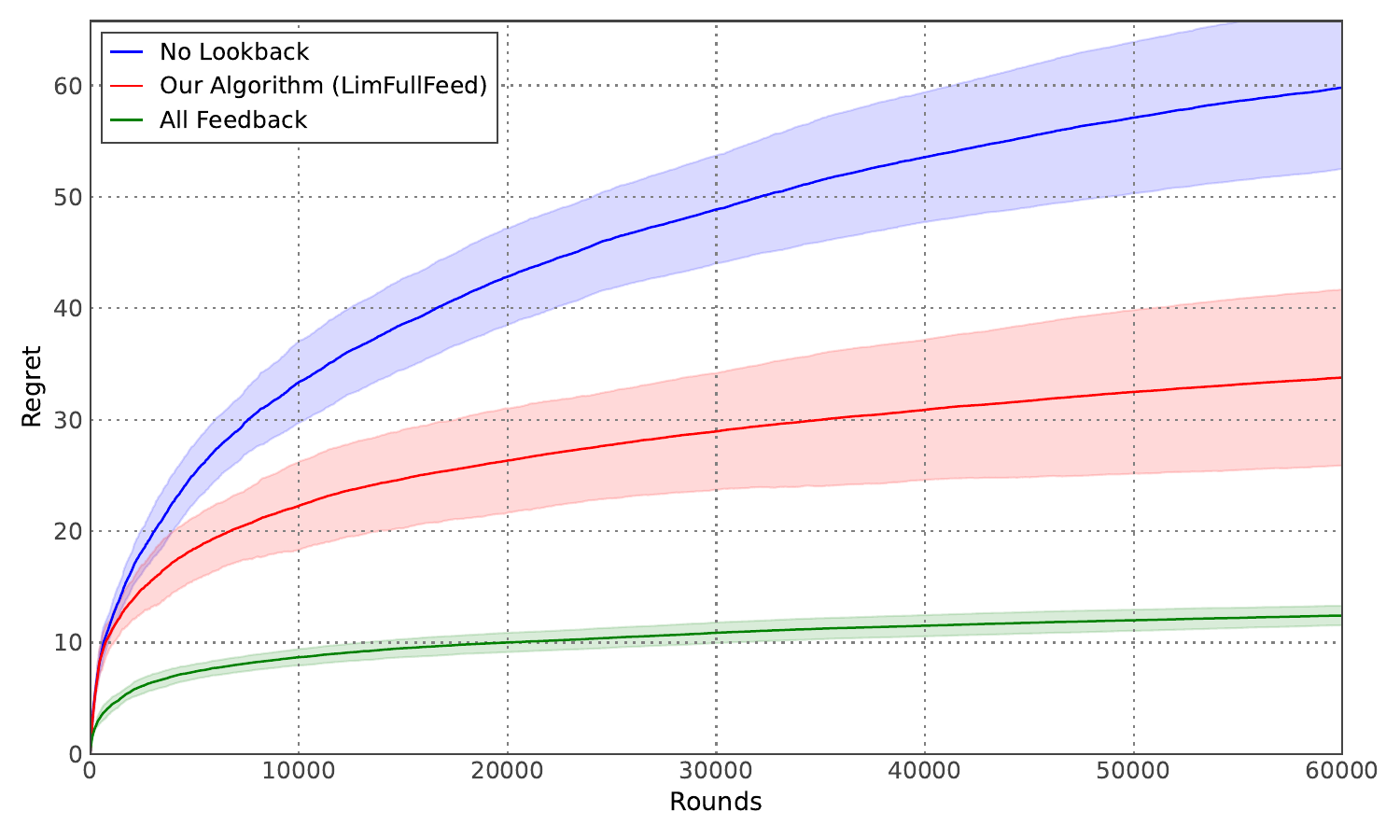}}
    \label{fig:full_information}
    \hfill
    \subfigure[Bandit Setting]{
    \includegraphics[scale=0.3]{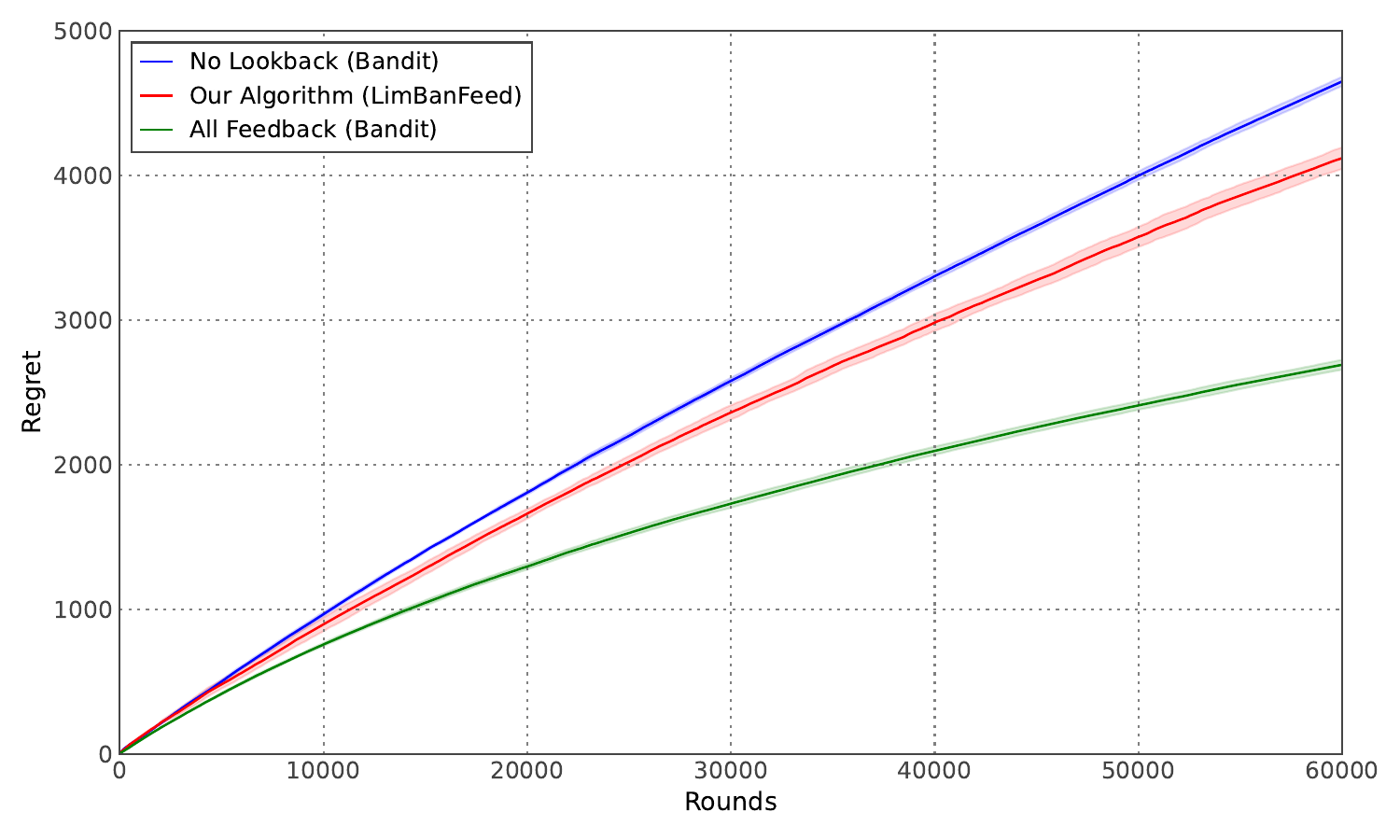}
    \label{fig:bandit}
    }
    \caption{\small 
    Results comparing our approaches to the other methods (RouterBench):
    (a) \nolookback (evaluates prompt at round when feedback is requested) and (b) \allfeedback (observes feedback at all rounds). Clearly, \limfullfeed has better regret guarantees than \nolookback by careful choice of feedback. \allfeedback suffers the smallest regret due to more data.
    }
    \label{fig:results-router-bench-full-info}
\end{figure*}
There is a gap of \( d/\sqrt{K} \) between the upper and lower bounds in the bandit setting. This arises because the lower-bound analysis reduces to hypothesis testing in a \(K\)-dimensional subspace. If we were to ignore the feature structure and treat the experts as arms in a standard multi-armed bandit setting, we could establish a lower bound of \( \Omega(T\sqrt{K/m}) \), which is looser by a factor of \( \sqrt{K} \).  
Similar to the full-information setting, the upper bound includes an additional factor of \( \sqrt{d} \) due to the algorithm’s lack of knowledge about the prompt vectors, requiring a union bound over all vectors in the \( d \)-dimensional space. This extra factor can be eliminated if the algorithm has prior knowledge of the prompts. The remaining gap of \( \sqrt{d/K} \) persists because the lower bound analysis is based on hypothesis testing with \( K \) vectors and may be further improved. It will be an interesting direction of future work to improve the lower bound in the bandit setting.

\section{Experiments}
\label{sec:experiments}
We evaluate our approach on large-scale LLM routing benchmarks.

\paragraph{Datasets and Baselines.}
We consider two widely used routing benchmarks: RouterBench\citep{hu2024routerbench} and Nectar\citep{zhu24starling7b}. Both datasets provide prompt-level evaluations across multiple LLM experts, enabling controlled simulation of online routing with ground-truth rewards. We follow a standard protocol and sample a subset of prompts to simulate an online interaction stream. 

We compare against two representative baselines. \textbf{NoLookBack} requests feedback only for the current round when querying, without leveraging past data. This corresponds to a naive strategy that ignores the structure of the problem. \textbf{AllFeedback} assumes access to feedback at every round and serves as a performance upper bound.

We simulate an online routing process over a fixed horizon $T$, with a feedback budget $m \ll T$. At each round, the algorithm selects an expert based on the observed context, and feedback is collected according to the algorithm's strategy. We report cumulative regret with respect to the best expert in hindsight. All results are averaged over 5 runs. Details are provided in Appendix~\ref{sec:exp_details}.

\begin{table*}[t]
  \centering
  \caption{Cumulative regret with respect to feedback budget $m$
  on Nectar ($K{=}6$, $d{=}40$, $T{=}60000$).}
  \label{tab:mscaling}
  \begin{tabular}{r  rr  rr}
    \toprule
    & \multicolumn{2}{c}{Full-information} &
      \multicolumn{2}{c}{Bandit} \\
    \cmidrule(lr){2-3}\cmidrule(lr){4-5}
    $m$ &
    LimFullFeed & NoLookBack &
    LimBanFeed & NoLookBack \\
    \midrule
    1000  & $672.9 \pm 37.4$ & $\mathbf{631.6 \pm 22.1}$ &
               $\mathbf{3583.4 \pm 68.0}$ & $3789.6 \pm 30.4$ \\
    2000  & $441.4 \pm 42.6$ & $\mathbf{408.0 \pm 26.2}$ &
               $\mathbf{3084.1 \pm 25.6}$ & $3580.0 \pm 73.3$ \\
    5000  & $\mathbf{214.1 \pm 7.7}$ & $218.7 \pm 10.6$ &
               $\mathbf{2295.7 \pm 42.3}$ & $3020.7 \pm 17.2$ \\
    10000 & $\mathbf{128.0 \pm 8.7}$ & $129.0 \pm 11.0$ &
               $\mathbf{1839.1 \pm 29.1}$ & $2538.4 \pm 22.2$ \\
    20000 & $\mathbf{73.6 \pm 6.0}$ & $75.8 \pm 6.1$ &
               $\mathbf{1567.9 \pm 13.4}$ & $2090.7 \pm 21.2$ \\
    \bottomrule
  \end{tabular}
\end{table*}

\paragraph{Main Results.}
Figure~\ref{fig:results-router-bench-full-info} and ~\ref{fig:combined} show the regret as a function of the number of rounds in both the full-information and bandit settings on RouterBench and Nectar. We observe that \limfullfeed consistently outperforms \nolookback. Interpreting regret as the cumulative number of suboptimal decisions, the gap between the two methods becomes substantial over time. In particular, by the end of the horizon ($T = 60000$), \limfullfeed achieves noticeably lower regret than \nolookback, indicating a significantly smaller number of mistakes. This improvement highlights the importance of selecting informative prompts when querying feedback. By looking back and choosing past observations that maximize information gain, \limfullfeed reduces uncertainty more efficiently and accelerates learning. In contrast, \nolookback queries feedback only at the current round, which leads to less informative data collection and slower error reduction. We also observe that the variance of \allfeedback is consistently smaller across rounds. This is expected, as \allfeedback has access to substantially more feedback, resulting in more stable estimates and reduced variability across runs.

The same trend holds in the bandit setting, where \limbanfeed consistently achieves lower regret than \nolookback, with the gap increasing over time. As in the full-information setting, \allfeedback exhibits lower variance due to the larger amount of available data.

Finally, regret in the full-information setting is uniformly lower than in the bandit setting. This aligns with our theoretical analysis, as the full-information setting provides richer feedback per observation, enabling faster reduction of uncertainty.

\textbf{Effect of the Feedback Budget.}
Table~\ref{tab:mscaling} shows that regret decreases as $m$ increases, consistent with the $\sqrt{T/m}$ dependence in our theory. In the bandit setting, \limbanfeed consistently outperforms \nolookback, with the gap widening for larger $m$. In the full-information setting, the gap is smaller for small $m$, since each feedback reveals all experts, but \limfullfeed becomes competitive as $m$ increases.

\textbf{Runtime Analysis.}
Table ~\ref{tab:runtime} shows that both methods scale approximately linearly with T. \limbanfeed scales linearly with K, while \limfullfeed is largely independent of K due to its shared covariance structure.

\section{Conclusions}
\label{sec:conclusions}

In this paper, we tackled adaptive prompt routing to LLM experts in an online learning setting with limited feedback. We framed this as a contextual bandit problem and introduced new algorithms for both full-information and bandit feedback settings. Our key innovation is showing when and where to observe feedback, which is expert-dependent in the bandit setting. Theoretical guarantees on regret show near-optimal learning rates, and empirical evaluations confirm our methods' effectiveness. Future work could explore dynamic expert availability, varying prompt distributions, and lower bounds in the bandit setting to refine regret minimization strategies.

\section*{Limitations}
We assume that feedback for past prompts can be obtained retrospectively, which may require storing model outputs and relying on noisy human or LLM-based judgments. Our experiments use static benchmarks to simulate online routing, which may not fully capture non-stationary user behavior or deployment constraints.

\bibliography{acl}

\appendix
\clearpage
\onecolumn
\section*{Appendix}
\setlength{\jot}{4pt}
\section{Bandit Feedback with Variable Costs}
\label{sec:variable-cost bandit}

We can also extend Algorithm \ref{alg:banditv2} to the setting where the $K$ experts have different costs of being evaluated and there is a budget on the overall cost. More precisely, suppose the total cost budget is $m$ and the cost of evaluation of the $j^{\s{th}}$ expert is $z_j$. Without loss of generality, we can assume that $z_1 \ge z_2 \ge \dots \ge z_{K}$. For the expert $j$ with cost $z_j$, we evaluate it after it has been used $z_j \cdot \widetilde{z}$ times for some $\widetilde{z}>0$. 

In other words, the experts is evaluated with a frequency that is proportional to their cost. In Algorithm \ref{alg:banditv2}, each expert had the same cost and therefore, the strategy of evaluation is identical for every expert. Now, given the budget on the feedback rounds $m$, $\widetilde{z}$ can be computed as follows: Again, recall that the 
 number of rounds where the expert $a$ has been used to generate a response in the entire time horizon is $\left|S_{T,a}\right|$. 
 In that case, the number of times the $a^{\s{th}}$ expert has been evaluated is $\lfloor \left|S_{T,a}\right|/(z_j \cdot \widetilde{z}) \rfloor$. We need the total cost to be smaller than $m$, hence, 
\begin{align*}
   &\sum_j \lfloor \frac{\left|S_{T,a}\right|}{z_j \cdot \widetilde{z}} \rfloor \cdot z_j  \le \sum_j \frac{\left|S_{T,a}\right|}{z_j \cdot \widetilde{z}} \cdot z_j \\
   &= \sum_j \frac{\left|S_{T,a}\right|}{\widetilde{z}} = \frac{T}{\widetilde{z}} \le m \implies \widetilde{z} \ge \frac{T}{m}.
\end{align*}
Now the entire analysis proceeds in the same way as the proof for Theorem \ref{thm:bandit}. The only difference lies in the non-uniform rate at which each expert is updated. More precisely, instead of a common interval length of $z$ for every expert $a\in [K]$ as in Algorithm \ref{alg:banditv2}, now we have an interval length of $z_a\cdot \widetilde{z}$ for the $a^{\text{th}}$ expert. Going through the same set of calculations, we can show that the regret guarantee in this case is going to be 
\begin{align*}
    \s{Reg}(T) = O\Big(Tm^{-1/2}\sqrt{d\beta\Big(\sum_j z_j\Big)\log\Big(\frac{d\lambda+m}{d}\Big)}\Big)
\end{align*}
Substituting the value of $\beta_{m/K}$ as in Theorem \ref{thm:bandit}, the final regret guarantee can be written as $$\s{Reg}(T) = O(dTm^{-1/2}(\sum_j z_j)^{1/2}\log T).$$
Notice that when all the costs are $1$, we get the same regret guarantee as in Theorem \ref{thm:bandit}. Further, as is common in this framework, the additional cost of limited feedback comes in the form of a multiplicative factor of $\sqrt{T/m}$.

\section{Detailed Proofs}\label{sec:proofs}

\subsection{Proof of Theorem \ref{thm:full_information}}\label{app:detailed_full_info}

\begin{proof}[Proof of Theorem \ref{thm:full_information}]

\textit{Notations:} We start by setting up some notations. Let $O_{t} \subseteq [t]$ denote the set of rounds until (including) round $t$ in which the algorithm chooses to get feedback. Note that $O_t \subseteq O_{t'}$ for all $t' \ge t$. For the set of prompts evaluated until round $t$, recall $S_t$ to be the set of past rounds when the prompts were provided as input to the algorithm.  Let us denote $\tau_{t} \triangleq \max \{\ell \in [t - 1]: \ell \in O_{t} \wedge (\ell - 1) \notin O_{t}\}$ denote the last round before round $t$ when the experts were updated. We denote $a \wedge b = \min(a,b)$.

We can decompose the RHS in \eqref{eq:reg_simple} as follows:

\begin{align*}
    \sum_{t=1}^{T} \ip{\theta_{a^{\star}_t},x_t} - \sum_{t=1}^{T} \ip{\theta_{a_t},x_t} = \sum_{t=1}^{T} \ip{\theta_{a^{\star}_t} -\widehat{\theta}_{t,a^{\star}_t},x_t} + \sum_{t=1}^{T} \ip{\widehat{\theta}_{t,a^{\star}_t}-\widehat{\theta}_{t,a_t},x_t} + \sum_{t=1}^{T} \ip{\widehat{\theta}_{t,a_t}-\theta_{a_t},x_t}.
\end{align*} 

Note that the second term is negative. This is because, by choice of our algorithm, we have picked the expert index $a_t$ at each round $t$ such that $\ip{\widehat{\theta}_{t,a^{\star}_t}-\widehat{\theta}_{t,a_t},x_t} \le 0$. Therefore, we can ignore the second term when bounding the LHS from above and then apply the Cauchy Schwarz inequality

Note that the covariance matrices for all experts remain identical throughout the time horizon. We  denote the covariance matrix of all experts at the beginning of round $t$ by $V_{t-1}$. Recall that feedback is solicited at every interval of $T/m$ rounds. Hence, the covariance matrix for all experts remains unchanged between two distinct feedback collections. From standard guarantees on confidence sets regarding the OLS estimator in the linear model (see Chapter 19 in \cite{lattimore2020bandit}), we have the following at every round $t\in [\s{T}]$ and every expert $a\in [K]$ with probability at least $1-o(T^{-2})$: 
\begin{align}\label{eq:high_prob}
    \lr{\theta_t - \widehat{\theta}_{t,a}}_{V_{t-1}} \le \beta  \text{ where } \beta = \sqrt{\lambda}+\sqrt{6\log T+d\log \Big(\frac{d\lambda+mL^2}{d\lambda}\Big)}.
\end{align}
Denote the event $\ca{E}$ where \eqref{eq:high_prob} is true for all rounds $t\in [T]$ and all experts $a\in [K]$.
 Moving forward, we can now condition on the above high probability event $\ca{E}$. 
 We now show the following for any round $t$:
\begin{align}\label{eq:conf}
     \ip{\theta_{a^{\star}_t},x_t} -  \ip{\theta_{a_t},x_t} \le  \lr{\theta_{a^{\star}_t} -\widehat{\theta}_{t,a^{\star}_t}}_{V_{t-1}}\lr{x_t}_{V_{t-1}^{-1}}+ \lr{\widehat{\theta}_{t,a_t}-\theta_{a_t}}_{V_{t-1}}\lr{x_t}_{V_{t-1}^{-1}} \le  2\beta  \lr{x_t}_{V_{t-1}^{-1}} = 2\beta  \lr{x_t}_{V_{\tau_t}^{-1}}.
\end{align} 
The last step follows from the fact that by definition $\tau_{t,a_t}$ is the last round (before round $t$) when the expert $a_t$ had been updated. 
Note that for each round $t$, we also have $\ip{\theta_{a^{\star}_t},x_t} - \ip{\theta_{a_t},x_t} \le 1$ from the assumption in theorem statement.  Therefore, we can combine to show
\begin{align*}
    \sum_{t=1}^{T} \ip{\theta_{a^{\star}_t},x_t} - \sum_{t=1}^{T} \ip{\theta_{a_t},x_t} \le 2\beta\sum_{t=1}^{T}  \Big(1 \wedge\lr{x_t}_{V_{t-1}^{-1}}\Big) \le 2\sqrt{T\beta\sum_{t=1}^{T}\Big(1 \wedge\lr{x_t}^2_{V_{t-1}^{-1}}\Big)}
\end{align*} 
where the final step follows from application of Cauchy-Schwarz inequality. Now, continuing from \eqref{eq:conf}, for any round $t$, we use the fact that for any $u\ge 0$, we have $u \wedge 1 \le 2\log(1+u)$:
\begin{align*}
    1 \wedge\lr{x_t}^2_{V_{\tau_t}^{-1}} &\le  \log \Big(1+\lr{x_t}^2_{V_{\tau_t}^{-1}}\Big) =  \log \s{det}\Big(I+V_{\tau_t}^{-1/2}x_tx_t^{T}V_{\tau_t}^{-1/2}\Big) \\
    &=  \log \s{det}\Big(V_{\tau_t}+x_tx_t^{T}\Big) -  \log \s{det} V_{\tau_t} 
\end{align*}
Now, consider a round $t\in O_T$ where feedback has been observed for some prompt in the history. As we defined before, $\tau_t$ corresponds to the previous round in history where feedback has been observed. From Algorithm \ref{alg:full_info}, recall that $s_t$ was the prompt chosen at round $t\in O_T$ for feedback following which we had updated $V_{t-1}=V_{\tau_t}+x_{s_t}x_{s_t}^{T}$. Further, recall that the prompt index $s_t$ at round $t$ was chosen in the following way:
\begin{align*}
    s_t = \s{argmax}_{\ell \in \, [t]\setminus \mathcal{S}_{t-1}\, } \s{det}( V_{t-1} + x_{\ell}x_{\ell}^{T}) = \s{argmax}_{\ell \in \, [t]\setminus \mathcal{S}_{t-1}\, } \s{det}( V_{\tau_t} + x_{\ell}x_{\ell}^{T})
\end{align*}
 implying that the prompt $s_t$ increases the determinant of the covariance matrix $V_{\tau_t}$ the most. Thus, for every round $r\in [\tau_t+1,t]$, we must have $\s{det}(V_{r} + x_{r}x_{r}^{T}) = \s{det}(V_{\tau_t}+x_{r}x_{r}^{T}) \le \s{det}(V_{\tau_t}+x_{s_t}x_{s_t}^{T})$. Hence, we have
 \begin{align*}
     \sum_{r\in [\tau_t+1,t]} \log \s{det}(V_{r} + x_{r}x_{r}^{T}) \le \lceil T/m \rceil \log \s{det}(V_{\tau_t}+x_{s_t}x_{s_t}^{T}) = \lceil T/m \rceil \log \s{det} V_{t-1}.
 \end{align*}
Therefore, using the fact that the covariance matrix remains unchanged for $T/m$ rounds, we can write 
\begin{align*}
    \sum_{t=1}^{T}\Big(1 \wedge\lr{x_t}^2_{V_{\tau_t}^{-1}}\Big) &\le \lceil T/m \rceil \sum_{t\in O_{T}} \Big(\log \s{det} V_{t} - \log \s{det} V_{\tau_t}\Big)+Tm^{-1} \\
    &= \lceil T/m \rceil \Big(\log \s{det} V_{T} - \log \s{det} V_{0}\Big)+Tm^{-1}.
\end{align*}

Notice that $V_{t-1}=\sum_{t\in O_T} x_{s_t}x_{s_t}^{T}$, $\left|O_T\right|=m$ (since we are getting feedback for $m$ rounds). Moreover, for all prompts $\{x_t\}_{t\in [T]}$, we have $\s{Tr}(x_tx_t^T) = \lr{x_t}_2^2 \le 1$ since all prompts are within the unit ball. 
Now, we use the AM-GM inequality and linearity of trace operation to show 
\begin{align*}
    \s{det} V_{T} \le \Big(d^{-1} \s{Tr}(V_{T})\Big)^d \le \Big(d^{-1} \Big(\s{Tr}(V_0)+m\Big)\Big)^d  
\end{align*}
We have $\s{Tr}(V_0) \le d\lambda$ and therefore, we bound the regret conditioned on the event $\ca{E}$ (denote by $\s{Reg}(T)\mid \ca{E}$) as 
\begin{align*}
    \s{Reg}(T)\mid \ca{E} = O\Big(Tm^{-1/2}\sqrt{d\beta\log\Big(\frac{d\lambda+m}{d}\Big)}\Big)
\end{align*}
When the event $\ca{E}$ is false, then the regret can be bounded by the $2T$ (worst-case). Therefore the final regret can be written as 
\begin{align*}
    \s{Reg}(T) = \s{Reg}(T)\mid \ca{E} + 2T\Pr(\ca{E}^c) = \s{Reg}(T)\mid \ca{E}+o(T^{-1}).
\end{align*}
Now, we can substitute the value for $\beta$ to obtain the final theorem statement.
\end{proof}

\subsection{Proof of Theorem \ref{thm:bandit}}\label{app:bandit}

\begin{proof}[Proof of Theorem \ref{thm:bandit}]

\textit{Notations:} We start by introducing new notations. For any expert $a\in [K]$, let $O_{t,a} \subseteq [t]$ denote the set of rounds until round $t$ in which the algorithm has opted to obtain a feedback for expert $a$. As before, for any $a\in [K]$, we have $O_{t,a} \subseteq O_{t',a}$ for all $t' \ge t$. Similar to full-information, for the set of prompts evaluated until round $t$, define $S_t$ to be the set of past rounds when the prompts were provided as input to the algorithm.  We denote $\tau_{t,a} \triangleq \max \{\ell \in [t - 1]: \ell \in O_{t,a} \wedge (\ell - 1) \notin O_{t,a}\}$ denote the last round before round $t$ when the expert $a$ was updated.

As in the proof of Theorem \ref{thm:full_information}, we will exploit standard guarantees on confidence sets regarding the OLS estimator in the linear model. We have the following at every round $t\in [\s{T}]$ and every expert $a\in [K]$ with probability at least $1-o(T^{-2})$: 
\begin{align}\label{eq:high_prob2}
    \lr{\theta_t - \widehat{\theta}_{t,a}}_{V_{t-1,a}} \le \beta  \text{ where } \beta = \sqrt{\lambda}+\sqrt{6\log T+d\log \Big(\frac{d\lambda+mL^2}{d\lambda}\Big)}.
\end{align}
Denote the event $\ca{E}$ where \eqref{eq:high_prob2} is true for all rounds $t\in [T]$ and all experts $a\in [K]$. For any prompt $x_t$ at round $t$, the parameter vector in the confidence region that leads to the highest reward is given by 
\begin{align*}
    \widetilde{\theta}_{t,a} = \sup_{\phi \in \lr{\phi - \widehat{\theta}_{t,a}}_{V_{t-1,a}} \le \beta } \ip{\phi, x_t}. 
\end{align*}
and therefore, the expert $a_t$ to generate the response is chosen as $\max_{a\in [K]} \ip{\widetilde{\theta}_{t,a},x_t}$. It is interesting to note that the bilinear optimization problem  has a nice closed form expression given by (see Sec 19.3.1 in \cite{lattimore2020bandit})
\begin{align*}
    a_t = \displaystyle \argmax_{a \in [K]} \ip{\widehat{\theta}_{t,a}, x_t} + \beta \cdot \lr{x_t}_{V_{t-1,a}^{-1}}
\end{align*}
    Now, we can decompose the RHS in  \eqref{eq:reg_simple} for any round $t\in [T]$ as follows:

\begin{align*}
    \ip{\theta_{a^{\star}_t},x_t} -  \ip{\theta_{a_t},x_t} \le  \ip{\widehat{\theta}_{a^{\star}_t},x_t}+\beta \lr{x_t}_{V_{t,a^{\star}_t}^{-1}}-\ip{\theta_{a_t},x_t} \le \ip{\widehat{\theta}_{a_t},x_t}+\beta \lr{x_t}_{V_{t,a_t}^{-1}}-\ip{\theta_{a_t},x_t}
\end{align*}

For the first inequality we bound the reward for the expert $a_t^{\star}$ from above by the upper confidence bound on the reward. The second inequality follows from $\ip{\widehat{\theta}_{t,a^{\star}_t}-\widehat{\theta}_{t,a_t},x_t} \le 0$. This is because, by choice of our algorithm, we have picked the expert index $a_t$ at each round $t$ having the largest upper confidence reward on prompt $x_t$.

 The covariance matrices for each of the experts are updated separately in the bandit feedback setting. Recall that the noisy reward for only a single expert is observed whenever the algorithm opts for a feedback. We denote the covariance matrix of expert $a\in [K]$ at round $t$ by $V_{t-1,a}$. Again, recall that feedback is solicited for expert $a$ after the expert $a$ has been used to generate a response $z$ times.  Now, we can show the following for any round $t$

\begin{align*}
 \ip{\theta_{a^{\star}_t},x_t} -  \ip{\theta_{a_t},x_t} &\le  \beta \lr{x_t}_{V_{t,a_t}^{-1}}+ \lr{\widehat{\theta}_{t,a_t}-\theta_{a_t}}_{V_{t,a_t}}\lr{x_t}_{V_{t,a_t}^{-1}} \\
 &\le 2\beta    \lr{x_t}_{V_{t,a_t}^{-1}}
\end{align*}

Now, we need to bound from above the final term on the RHS in terms of the individual experts as follows:
\begin{align*}
    \sum_{t=1}^{T}  \lr{x_t}_{V_{t,a_t}^{-1}} &= \sum_{a=1}^{K} \sum_{t=1}^{T}  \lr{x_t}_{V_{t,a_t}^{-1}} \mathbf{1}[a_t=a] \\
    &\le \sum_{a=1}^{K}\sum_{t=1}^{T}\Big(1 \wedge \lr{x_t}_{V_{t,a_t}^{-1}} \mathbf{1}[a_t=a]\Big) \\
    &\le \sqrt{T\sum_{a=1}^{K}\sum_{t=1}^{T}\Big(1 \wedge \lr{x_t}^2_{V_{t,a_t}^{-1}}\Big) \mathbf{1}[a_t=a]\Big)}
\end{align*} 
In the pre-final step, we used that in each round $t$, we also have $\ip{\theta_{a^{\star}_t},x_t} - \ip{\theta_{a_t},x_t} \le 1$ from the assumption in theorem statement. In the final step, we used the Cauchy-Schwarz inequality. Note that there are only $T$ terms in the RHS since at each round, only a single expert is chosen.
Consider a particular expert $a\in [K]$ - it is chosen for obtaining feedback after being used for generating responses $z$ times since the last time feedback was obtained for expert $a$. 

At any round $t$, recall $\tau_{t,a}= \max \{\ell \in [t - 1]: \ell \in O_{t, a} \wedge (\ell - 1) \notin O_{t, a}\}$ denote the last round from $t$ when the expert $a$ was chosen for obtaining feedback (after the counter for expert $a$ reached $z$). For any round $t\in O_{t,a}$, as per our notation, recall that $V_{t-1,a}$ is the covariance matrix for the expert $a$ at round $t$. Then $\tau_{t-1,a}$ is the previous round when the expert $a$ was updated. For any round $t$, we can use the fact that for any $u\ge 0$, we have $u \wedge 1 \le 2\log(1+u)$,
\begin{align*}
    1 \wedge\lr{x_t}^2_{V_{t,a_t}^{-1}}  &\le  \log \Big(1+\lr{x_t}^2_{V_{t,a_t}^{-1}}\Big) =  \log \s{det}\Big(I+V_{t,a_t}^{-1/2}x_tx_t^{T}V_{t,a_t}^{-1/2}\Big) \\
    &=  \log \s{det}\Big(V_{t,a_t}+x_tx_t^{T}\Big) -  \log \s{det} V_{t,a_t} =  \log \s{det}\Big(V_{\tau_{t,a_t},a_t}+x_tx_t^{T}\Big) -  \log \s{det} V_{\tau_{t,a_t},a_t}
\end{align*}

Now for a fixed expert $a\in [K]$, consider a round $t\in O_{T,a}$ where feedback has been observed for some prompt in the history of expert $a$. As we defined before, $\tau_{t-1,a}$ corresponds to the previous round in history (relative to round $t)$ where feedback has been observed for expert $a$. From Algorithm \ref{alg:banditv2}, recall that $s_t$ was the prompt chosen at round $t\in O_{t,a}$ for feedback following which we had updated $V_{t-1,a}=V_{\tau_{t-1},a}+x_{s_t}x_{s_t}^{T}$. Further, recall that the prompt index $s_t$ at round $t$ was chosen in the following way:
\begin{align*}
    s_t = \displaystyle \argmax_{s\in  \{s \in [t]: a_s = a\} \setminus S_{t - 1,a}} \s{det}( V_{t-1,a} + x_{s}x_{s}^{T}) = \displaystyle \argmax_{s\in  \{s \in [t]: a_s = a\} \setminus \mathcal{S}_{t-1,a}\, } \s{det}( V_{\tau_{t-1,a},a} + x_{x}x_{s}^{T})
\end{align*}
 implying that the prompt $s_t$ was responded to by the expert $a_t$ and further, the prompt $s_t$ increases the determinant of the covariance matrix $V_{\tau_{t-1,a},a}$ the most. Thus, for every round $r\in [\tau_{t-1,a}+1,t]$ satisfying $\mathbf{1}[a_r=a]$, we must have $\s{det}(V_{r,a} + x_{r}x_{r}^{T}) = \s{det}(V_{\tau_{t-1,a},a}+x_{r}x_{r}^{T}) \le \s{det}(V_{\tau_{t-1,a},a}+x_{s_t}x_{s_t}^{T})$. Hence, we have
 \begin{align*}
     \sum_{r\in [\tau_{t-1}+1,t]} \log \s{det}(V_{r,a} + x_{r}x_{r}^{T}) \le c \cdot \log \s{det}(V_{\tau_{t-1,a},a}+x_{s_t}x_{s_t}^{T}) = \lceil T/m \rceil \log \s{det} V_{t-1,a}.
 \end{align*}
 In the last step, we used the threshold on $z$ that would meet the feedback budget $m$. As we had shown, it is sufficient to have $z\ge T/m$ so that the total amount of feedback rounds is at most $m$.
Therefore, by combining all of these, we get 

\begin{align}\label{eq:interval}
\sqrt{T\sum_{a=1}^{K}\sum_{t=1}^{T}\Big(1 \wedge \lr{x_t}^2_{V_{t,a_t}^{-1}}\Big) \mathbf{1}[a_t=a]\Big)} &= \sqrt{T\sum_{a=1}^{K}\Big(z \sum_{t\in O_{T,a}} \Big(\log \s{det} V_{t-1,a} - \log \s{det} V_{\tau_{t-1,a},a}\Big)+c\Big)} \\
&\le Tm^{-1/2} \sqrt{\sum_{j=1}^{K} \Big(\log \s{det}(V_{t-1,a} -  \log \s{det} V_{0,a}\Big)}+\sqrt{cKT}.
\end{align}

For a fixed expert $a\in [K]$, notice that $V_{t-1,a}=\sum_{t\in O_{T,a}} x_{s_t}x_{s_t}^{T}$, $\left|O_T\right| \le m$ (since we are getting feedback for at most $m$ rounds for any expert). Moreover, for all prompts $\{x_t\}_{t\in [T]}$, we have $\s{Tr}(x_tx_t^T) = \lr{x_t}_2^2 \le 1$ since all prompts are within the unit ball. 
Now, we use the AM-GM inequality and linearity of trace operation to show 
\begin{align*}
    \s{det} V_{t-1,a} \le \Big(d^{-1} \s{Tr}(V_{t-1,a})\Big)^d \le \Big(d^{-1} \Big(\s{Tr}(V_{0,a})+m\Big)\Big)^d  
\end{align*}
We have $\s{Tr}(V_{0,a}) \le d\lambda$ and therefore, we bound the regret conditioned on the event $\ca{E}$ (denote by $\s{Reg}(T)\mid \ca{E}$) as 
\begin{align*}
    \s{Reg}(T)\mid \ca{E} = O\Big(TK^{1/2}m^{-1/2}\sqrt{d\beta\log\Big(\frac{d\lambda+m}{d}\Big)}\Big)
\end{align*}
\begin{align*}
    \s{Reg}(T) = \s{Reg}(T)\mid \ca{E} + 2T\Pr(\ca{E}^c) = \s{Reg}(T)\mid \ca{E}+o(T^{-1}).
\end{align*}
Now, we can substitute the value for $\beta$ to obtain the final theorem statement.
 
\end{proof}

\subsection{Proof of Theorem \ref{thm:lower_bound_full_info}}

\begin{proof}[Proof of Theorem \ref{thm:lower_bound_full_info}] 
    Consider number of rounds $T$, a feedback budget of $m$ (full-information feedback) and all observations to be Gaussian random variables with unit variance. 
    While constructing our instances, let us denote the set of prompts we choose from to be $\ca{X} \subset \{-1,+1\}^d$ and the set of model features to be $\theta\subset \{-\delta,+\delta\}^d$ for some $\delta>0$. For any two vectors ($x,y$), we can define the Hamming distance $d_{h}(x,y)=\sum_{i=1}^{d} \mathbf{1}[\text{sign}(x_i)\neq \text{sign}(y_i)]$.

    Fix a vector $x\in \{-1,+1\}^d$.
    Now, we define an instance in the full-information expert selection setting as follows: for each of the $T$ rounds, a single prompt $x\in \ca{X}$ is going to be demonstrated and the set of expert parameters in this instance is given by $\theta\equiv \{\theta \in \{-\delta,+\delta\}^d \mid d_h(\theta,x) \le 1\}$. In other words, the set of allowed experts have feature embedding which are within a Hamming distance of $1$ from the prompt embedding vector $x$. Given this environment, . Recall that the expert chosen at round $t$ is denoted by $a_t$.
    Now, we define $d$ alternate learning instances as follows: in each alternate instance, a vector $x'$ satisfying $d_h(x,x')=1$ is demonstrated at all rounds while the set of model features $\ca{T}(x)$ remain the same. Therefore, the set of prompts across the $d+1$ instances is $\ca{X} \equiv \{y \in \{-1,+1\}^d \mid d_h(y,x) \le 1\}$. 
    Since the distinction between the instances is only in the prompt (the expert features $\theta$ remain same in all instances), we denote by $\bb{P}_{y}$ and $\bb{E}_{y}$ the probability and expectation of events under a fixed policy for instance defined by prompt $y\in \ca{X}$.

    Consider the environment defined by the prompt $x$. Let us define the following event which is true when the number of times the sign of the $i^{\s{th}}$ entry of the prompt and algorithm output (observation of history's actions) differs is at least $T/2$
    \begin{align*}
        \ca{E}_{x,i} = \mathbf{1}[\sum_{i=1}^{T} \text{sign}(x_i) \neq \text{sign}(a_{ti}) \ge T/2]
    \end{align*}
    and the corresponding probability to be $p_{x,i}=\bb{P}_{x}(\ca{E}_{x,i})$. Now consider the alternate instance with prompt $x'\in \ca{X}$  where $x'_j=x_j$ for all $j\neq i$ and $x'_i = -x_i$. Therefore, $x$ and $x'$ have a Hamming distance of $1$ and differ in sign at the index $i$. Note that the event $\ca{E}_{x',i}$ is the complement of the event $\ca{E}_{x,i}$.  We have 
    \begin{align*}
        p_{x,i}+p_{x',i} = \bb{P}_{x}(\ca{E}_{x,i})+ \bb{P}_{x'}(\ca{E}_{x,i}^c) \ge \frac{1}{2}\exp\Big(-\s{D}_{\s{KL}}(\bb{P}_{x}||\bb{P}_{x'})\Big) 
    \end{align*}
    Above, we used the Bretagnolle Huber inequality (see \cite{lattimore2020bandit}) and $\s{D}_{\s{KL}}$ is the Kullback-Leibler Divergence. It is given by 
    \begin{align*}
        \s{D}_{\s{KL}}(\bb{P}_{x}||\bb{P}_{x'}) = m \bb{E}_{x}\Big[\s{D}_{\s{KL}}\Big(\ca{N}(\theta \cdot x,I_{d+1})||\ca{N}(\theta \cdot x',I_{d+1})\Big]\Big) =m \lr{\theta(x-x')}_2^2 = m\delta^2(d+1).
    \end{align*}
    In that case, for the value of $\delta=1/\sqrt{m(d+1)}$, we get $p_{x,i}+p_{x',i} \ge \exp(-1)/2$. 
    Now, we want to extend the analysis to environments
defined by prompts $y\in \ca{X}\setminus\{x\}$. Fix such a prompt $y$ and define the event which is true when the number of times the sign of the $i^{\s{th}}$ entry of the prompt and the expert differs is at least $T/2$
    \begin{align*}
        \ca{E}_{y,i} = \mathbf{1}[\sum_{i=1}^{T} \text{sign}(y_i) \neq \text{sign}(a_{ti}) \ge T/2]
    \end{align*}
Without loss of generality, assume that the index where $y$ and $x$ differs in sign is $j\neq i$ (otherwise the alternate instance trivially becomes the environment defined by $x$). In that case, consider the alternate instance defined by prompt $y'$ such that $\fl{y'}the = x_s$ for all $s\neq i$ and $\fl{y'}_i = -x_i$ ($y'$ and $x$ differs in sign at the index $i$). Therefore, $y$ and $y'$ have a Hamming distance of $2$ and differ in sign at the indices $i,j$. However, note that the event $\ca{E}_{y',i}$ is the complement of the event $\ca{E}_{y,i}$. 
 By a similar analysis as above, we have
 \begin{align*}
     p_{y,i}+p_{y',i} \ge \frac{1}{2}\exp\Big(-\s{D}_{\s{KL}}(\bb{P}_{y}||\bb{P}_{y'})\Big) = \frac{1}{2} \exp\Big(-m \lr{\theta(y-y')}_2^2 \Big) \ge \frac{1}{2} \exp\Big(-4m\delta^2(d+1)\Big).
 \end{align*}
In that case, for the value of $\delta=1/\sqrt{m(d+1)}$, we have 
    \begin{align*}
        p_{y,i}+p_{y',i} \ge \exp(-4)/2.
    \end{align*}
    Since there are $d+1$ possibilities of the prompt $x$, we can have 
    \begin{align*}
        \frac{1}{\left|\ca{X}\right|} \sum_{x\in \ca{X}}\sum_{i=1}^{d} p_{x,i} =  \frac{1}{\left|\ca{X}\right|}\sum_{i=1}^{d}  \sum_{x\in \ca{X}}  p_{x,i} \ge \frac{d\exp(-4)}{4}.
    \end{align*}
    Hence, there must exist a prompt $y \in \ca{X}$ for which $\sum_{i=1}^{d} p_{y,i} \ge d\exp(-4)/4$. For this instance defined by the prompt $y$ and expert features $\theta$, we have 
    \begin{align*}
        \s{Reg}(T) &= \mathbb{E}_{y}\Big[\sum_{t=1}^{T} \langle y, \delta \cdot y-\fl{ a}_{t}\rangle \Big] \\
         &= \mathbb{E}_{y}\Big[\sum_{t=1}^{T} \langle y, \delta y-A_{t}\rangle \Big] = \mathbb{E}_{y}\Big[\sum_{t=1}^{T} \sum_{i=1}^{d} y_i \Big( \delta y_i-a_{ti}\Big) \Big] = 2\delta \mathbb{E}_{y}\Big[\sum_{t=1}^{T} \sum_{i=1}^{d} \mathbf{1}[\text{sign}(y_i)\neq \text{sign}(a_{ti})] \Big] \\
         &= 2\delta \sum_{i=1}^{d}\mathbb{E}_{y}\Big[\sum_{t=1}^{T}  \mathbf{1}[\text{sign}(y_i)\neq \text{sign}(a_{ti})] \Big] =  \delta T \sum_{i=1}^{d} \bb{P}_{yi} \ge \frac{\delta d T \exp(-4)}{4} \ge \frac{T\sqrt{d}\exp(-4)}{8\sqrt{m}}.
    \end{align*}
    This completes the proof of the theorem.
\end{proof}

\subsection{Proof of Theorem \ref{thm:lower_bound_bandit}}\label{app:lower_bound}

\begin{proof}[Proof of Theorem \ref{thm:lower_bound_bandit}] 
The proof of the lower bound for the bandit setting follows on similar lines as the proof in Theorem \ref{thm:lower_bound_full_info}. 
    Consider number of rounds $T$, a feedback budget of $m$ and number of experts to be $K$ and all observations to be Gaussian random variables with unit variance. Note that in this setting, for a particular prompt, feedback is only provided to the expert that was used to generate a response for that prompt. As before, we denote the set of prompts to be $\ca{X} \subset \{-1,+1\}^d$ and the set of experts features to be $\theta\subset \{-\delta,+\delta\}^d$ for some $\delta>0$. For any two vectors $x,y$, we can define the Hamming distance $d_{h}(x,y)=\sum_{i=1}^{d} \mathbf{1}[\text{sign}(x_i)\neq \text{sign}(y_i)]$.

    Fix a prompt vector $x$ to be the  $d$-dimensional all ones vector. For simplicity assume that $d$ is divisible by $\s{K}$.
Now, we define an instance in the bandit feedback setting as follows: for each of the $T$ rounds, the single prompt $x\in \ca{X}$ is going to be demonstrated at all rounds and the set of expert parameters is given by a fixed set of $K$ vectors from the set $\theta\equiv \{\theta \in \{-\delta,+\delta\}^d \mid d_h(\theta,x) \le 1\}$. Without loss of generality, let $\{\theta_{a}\}_{a\in [K]}$ be the set of parameter vectors which has its sign flipped from $x$ in one of the first $K$ entries. 

Now, we can define $K$ alternate learning instances as follows: in each alternate instance, a vector $x'$ satisfying $d_h(x,x')=1$ is demonstrated at all rounds while the set of expert features $\ca{T}(x)$ remain the same. Further $x'$ has its sign changed only among the first $K$ entries of $x$. Therefore the last $d-K$ indices are inconsequential and we basically have a $K$-dimensional problem. Now, we go through the same steps as in the proof of Theorem \ref{thm:lower_bound_full_info}. However, since we are in the bandit feedback setting, we will get after using the Bretagnolle Huber inequality that the KL divergence between data distributions of any two instances is $O(\exp(-m\delta^2))$. Hence, substituting $\delta=1/\sqrt{m}$ and then resuming the same steps gives us the following lower bound:
\begin{align*}
   \s{Reg}(\s{T}) \ge  \frac{TK\exp(-4)}{8\sqrt{m}}.
\end{align*}

\end{proof}

\twocolumn
\section{Related Work}
\label{sec:related work}

There are many notions of limited feedback and resource constraints that have been studied in the past. Partial monitoring \citep{agrawal89asymptotically,bartok12partial,bartok14partial} and learning to rank \citep{radlinski08learning,kveton15cascading,kveton15combinatorial,li16contextual} are bandit problems where the agent observes rewards of taken actions only partially. In our setting, the agent observes the reward fully but decides what to observe.
In bandits with delayed feedback \citep{zhou19learning,vernade20linear}, the agent observes rewards of taken action with delays. The delay is not controlled by the agent. Our agent observes rewards with delays but decides what to observe, and thus controls the delay. Further, it is well known that the linear model in bandit algorithms can be updated lazily, whenever the determinant of the covariance matrix increases significantly (Section 5.1 in \citet{abbasi-yadkori11improved}). When the model is updated, all past observations are used. In our setting, we update the model periodically but only use a subset of past observations selected by the agent. The knapsack is a popular way of modeling resource constraints in bandits \citep{tran-thanh12knapsack,agrawal16linear}. We have a resource constraint but it is significantly simpler. This is why we can minimize the regret greedily at a near-optimal rate by periodically taking the most uncertain action in the past. In conservative bandits \citep{wu16conservative,kazerouni17conservative}, the agent takes a greedy exploratory action after accumulating a sufficient exploratory budget. The regret bound involves an extra term due to taking the safe action. In our setting, the agent has more control because it can observe the reward of any past action. Therefore, the extra term does not appear. In online learning with abstention \citep{cortes18online}, the agent decides whether to predict or abstain. When the agent abstains, it pays a fixed a cost. We incur regret in each round and decide which past rewards to observe.

Recently, model selection under resource constraints has been studied in the context of LLMs. Most techniques aim to reduce the cost of inference by choosing models (LLMs) appropriately. The cost of inference depends on input and output token length and can be either API level cost or the compute cost of hosting and serving LLM. Note that our notion of cost is different - we consider the cost of labeling the LLM responses through human or other forms of gold feedback. 

A popular technique towards reducing costs is LLM cascading that invokes LLMs sequentially, progressing to higher cost LLMs, till the response is deemed satisfactory, often by another model \citep{chen2023frugalgpt, aggarwal2023automix, ramirez2024optimising}. More closely related to our bandit setting is LLM routing that aims to route queries directly to the appropriate model, requiring only a single inference \citep{shekhar2024towards,vsakota2024fly}. One major difference between these works and ours is that while we aim to learn the appropriate models in an online manner (either full information setting or bandit setting) by selectively obtaining labels for certain queries, existing works mostly aim to predict the output quality or performance (either absolute or relative to each other) of the LLMs for given queries for routing them \citep{dinghybrid, shekhar2024towards, shnitzer2023large, lu-etal-2024-routing, hari2023tryage, ong2024routellm, vsakota2024fly}. The routing is based on the prediction and other considerations such as cost and latency.  

Some recent works have modeled model selection with cost considerations as a multi armed bandit problem. MetaLLM \citep{nguyen2024metallm} frames the reward as a linear combination of accuracy and cost of inference and proposes a bandit algorithm for learning the best arm. Again, the notion of cost and the modeling of the problem are both different in our setting. Recently, TI-UCB \citep{xia2024llm} predicts the increase of model performances due to training or finetuning and efficiently balances exploration and exploitation in model selection. \cite{dai2024cost} proposed CSMAB-V based on combinatorial multi-armed bandits to select a good LLM combination, which aims to balance cost and rewards. This is also a different setting compared to ours. 

Other works such as \cite{huang2025thriftllm} focused on selecting a set of LLMs under a given cost budget to maximize performance. \cite{owodunni2023koya} proposed a recommender system approach called Koya for selecting the best LLM for a given task and language. There has also been work on online model selection with partial information. 
In particular, \cite{NEURIPS2019_433371e6} investigated model selection under contextual bandit feedback whereas \cite{pmlr-v139-cella21a} framed online model selection as a rested bandit problem. Further, \citet{karimi2021online} leverages active learning to identify the best model from a pool of pre-trained classifiers. Other recent work has focused on the orthogonal problem of leveraging LLMs for model selection \cite{wu2024large}. These works focus mostly on leveraging LLMs to capture the structural and semantic properties of the model for better selection.

Online learning and bandits with observation budget have been studied extensively in the adversarial setting. This line of work was started by \citet{hembolt97some}. \citet{cesabianchi05minimizing} proved a $T \sqrt{(\log K) / m}$ regret bound for the exponentially-weighted forecaster, where $K$ is the number of experts. \citet{audibert10regret} derived a $T \sqrt{K (\log K) / m}$ regret bound for the bandit setting. None of these works study the contextual setting. However, their scaling with the number of rounds $T$ and observation budget $m$ is similar to our work. The closest related work in the stochastic bandit setting is \citet{tucker23bandits}. This work proposes a contextual bandit algorithm but it is not analyzed. The algorithm maximizes the difference of rewards and costs, and is not applicable to our problem because the observation cost and reward are not comparable, and thus cannot be simply subtracted. We both propose a practical algorithm and analyze it.

\section{Implementation Details}
\label{sec:exp_details}
For our experiments, we use the RouterBench~\citep{hu2024routerbench} which is a high-quality routing benchmark consisting a diverse set of 11 Large Language Models. There are approximately $405k$ prompts and for each prompt, responses from the 11 models are evaluated and normalized to 0 and 1. We randomly chose a subset of $10k$ prompts out of the total pool of $405k$ prompts. We found no model dominates the others more than a third of the time (\cref{sec:introduction}). Thus, we have $K=11$ experts in this setting. For each of the $10k$ prompts, we create its $384$ dimensional embedding using the bge-small-en-v1.5 model \citep{bge_embedding}. Next we do PCA to reduce the dimension to $40$, capturing $50\%$ of the feature variance, resulting in a $10000 \times 40$ sized data matrix. In order to generate the expert parameters $\{\theta_a\}_{a\in [K]}$, we take the evaluated ranking from GPT-4 for each of $10k$ prompts and create a one-hot vector of dimension $11$. In this vector, only the model whose response has been selected as best is assigned $1$ and we have $0$ everywhere else. Therefore we have a response vector of dimension $10000\times 11$. Now we obtain the parameter $\{\theta_a\}_{a\in [K]}$ for each of the experts by computing the OLS estimate on this dataset.

\begin{table}[t]
  \centering
  \caption{Runtime in seconds with respect $K$, $d$, and $T$ on Nectar.}
  \label{tab:runtime}
  \begin{tabular}{l r  rr}
    \toprule
    Axis & Value & LimBanFeed & LimFullFeed \\
    \midrule
   {$K$}
      & 2  & $2.67 \pm 0.02$ & $1.73 \pm 0.34$ \\
      & 6  & $4.30 \pm 0.03$ & $1.49 \pm 0.01$ \\
      & 10 & $5.93 \pm 0.04$ & $1.47 \pm 0.01$ \\
      & 16 & $8.36 \pm 0.02$ & $1.49 \pm 0.04$ \\
      & 20 & $9.97 \pm 0.00$ & $1.47 \pm 0.00$ \\
    \midrule
   {$d$}
      & 8   & $4.25 \pm 0.00$ & $1.49 \pm 0.04$ \\
      & 24  & $4.24 \pm 0.00$ & $1.46 \pm 0.00$ \\
      & 40  & $4.27 \pm 0.04$ & $1.47 \pm 0.00$ \\
      & 96  & $4.34 \pm 0.00$ & $1.51 \pm 0.01$ \\
      & 128 & $4.34 \pm 0.00$ & $1.54 \pm 0.04$ \\
    \midrule
    {$T$}
      & 500    & $0.41 \pm 0.00$ & $0.15 \pm 0.00$ \\
      & 1{,}000 & $0.82 \pm 0.00$ & $0.29 \pm 0.00$ \\
      & 5{,}000 & $4.27 \pm 0.04$ & $1.47 \pm 0.00$ \\
      & 10{,}000 & $8.89 \pm 0.02$ & $2.98 \pm 0.05$ \\
      & 15{,}000 & $14.35 \pm 0.03$ & $4.47 \pm 0.03$ \\
    \bottomrule
  \end{tabular}
\end{table}

\begin{figure*}[t]
    \centering
    \subfigure[Full Information Setting]{
    \includegraphics[scale=0.30]{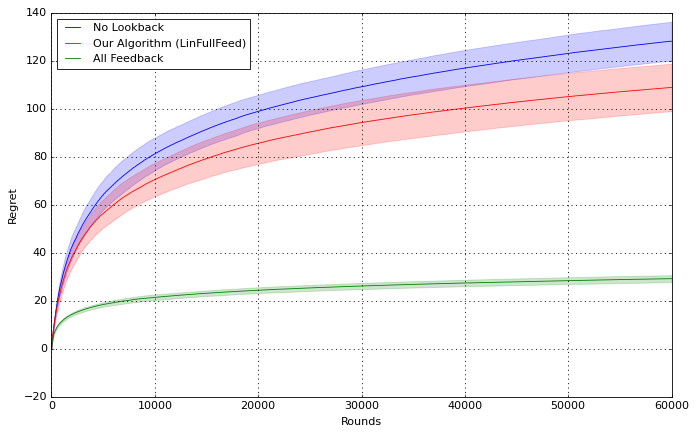}
    \label{fig:full_information}
    }
    \hfill
    \subfigure[Bandit Setting]{
    \includegraphics[scale=0.30]{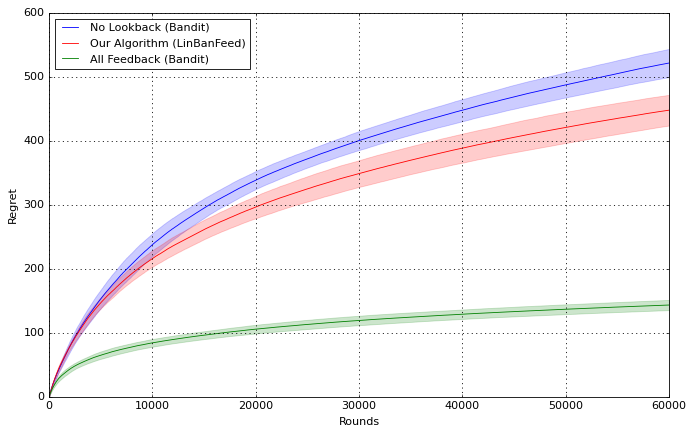}
    \label{fig:bandit}
    }

    \caption{\small Results comparing our approaches to the other methods (Nectar):
    (a) \nolookback (evaluates prompt at round when feedback is requested) and (b) \allfeedback (observes feedback at all rounds). Clearly, \limfullfeed has better regret guarantees than \nolookback by careful choice of feedback. \allfeedback suffers the smallest regret due to more data.
    }
    \label{fig:combined}
\end{figure*}

\end{document}